\documentclass[authoryear,5p]{elsarticle}              %\documentclass{elsart}
\usepackage{amssymb,amsmath,amsfonts,yfonts}
\usepackage{amsmath,mathrsfs,bm,url,times}
\usepackage{latexsym}
\usepackage{graphics,color}
\usepackage{graphicx,psfig,epsfig}
\usepackage{subfigure,float}
\usepackage{algorithmic,algorithm}
\usepackage{url}
\usepackage{epstopdf}
\usepackage{cases}

\usepackage{natbib}

\newtheorem{proposition}{Proposition}
\newtheorem{theorem}{Theorem}

\newtheorem{definition}{Definition}

\newtheorem{remark}{Remark}

\newenvironment{proof}{\vspace{1ex}\noindent{\bf Proof.}\hspace{0.5em}}{\hfill\qed\vspace{1ex}}
\def\qed{\hfill \vrule height 6pt width 6pt depth 0pt}

\newcommand{\m}{\text{\it m}}
\newcommand{\M}{\mathcal M}
\newcommand{\N}{\mathcal N}
\newcommand{\R}{{\mathbb R}}

\newcommand{\oneton}{1,\cdots,n}

\newcommand{\zerotoinfty}{0,1,2,\cdots}
\newcommand{\sign}{\text{sign}}

\newcommand{\ignore}[1]{}

\begin{document}
\begin{titlepage}
\begin{center}
{\large\bf  Centralized and Decentralized Global Outer-synchronization of Asymmetric Recurrent Time-varying Neural Network by Data-sampling}\footnote{This work is
jointly supported by the National Natural Sciences Foundation of
China under Grant Nos. 61273211 and 61273309, and the Program for New Century Excellent
Talents in University (NCET-13-0139)}
\\[0.2in]
\begin{center}

Wenlian Lu\footnote{
Wenlian Lu is with the Department of Radiology, JinLing Hospital of Nanjing, People¡¯s Republic of China, and also with the Jinling Hospital-Fudan University Computational Translational Medicine Center, Centre for Computational
Systems Biology, and School of Mathematical Sciences, Fudan University, People's Republic of China (wenlian@fudan.edu.cn)},
Ren Zheng\footnote{Ren Zheng  is with the School of Mathematics, Fudan University, 200433, Shanghai, China. (12110180051@fudan.edu.cn)},
Tianping Chen\footnote{Tianping Chen is with the School of Computer
Sciences/Mathematics, Fudan University, 200433, Shanghai, China. \\
\indent ~~Corresponding author: Tianping Chen. Email:
tchen@fudan.edu.cn}
\end{center}
\end{center}

\begin{center}
Abstract
\end{center}
In this paper, we discuss the outer-synchronization of the asymmetrically connected recurrent time-varying neural networks. By both centralized and decentralized discretization data sampling principles, we derive several sufficient conditions based on diverse vector norms  that guarantee that any two trajectories from
different initial values of the identical neural network system converge together.
The lower bounds of the common time intervals between data samples in centralized and decentralized
principles are proved to be positive, which guarantees exclusion of Zeno behavior.
A numerical example is provided to illustrate the efficiency of the theoretical results.


\end{titlepage}
\section{Introduction\label{Introduction}}

Recurrently connected neural networks, also known as the Hopfield neural
networks, have been extensively studied in past decades and found
many applications in different areas. Such applications heavily
depend on the dynamical behaviors of the system. Therefore, analysis
of the dynamics is a necessary step for practical design
of neural networks.

The dynamical behaviors of continuous-time recurrently asymmetrically
connected neural networks (CTRACNN) have been
studied at the very early stage of neural network research. For example,
multistable and oscillatory behaviors were studied by \citet{Amari,Amari1} and \citet{Wilson}.
Chaotic behaviors were studied by \citet{Som}.
\citet{Hopfield1,Hopfield2} studied stability of symmetrically connected networks
and showed their practical applicability to optimization problems. It should
be noted that Cohen and Grossberg, see \citet{Cohen} gave more rigorous
results on the global stability of networks.

The global stability of symmetrically connected networks described by differential equations has now been
well established. See \citet{Chen,Chen1,Chen2,Fang,Forti,Hirsch,Kasz,Kelly,Li,Matsuoka,Yang} and the references therein. More related to the present paper, the previous paper \citep{Liu} addressed the global self-synchronization
of general continuous-time asymmetrically connected recurrent networks and
discussed the independent identically distributed switching process on the selecting
the time-varying parameters in detail.

However, in applications, discrete iteration is popular to be employed to realize neural network process, rather than continuous-time equations.
Generally, synchronization analysis for differential equations cannot be applicable to the discrete-time situation. There are several papers \citep{Jin,Jin1,Wang} that discussed different types of discrete-time neural networks, where the step sizes were constants. However, in \citet{LiuChen,Manuel,Sey,WangLemmon}, these papers pointed out that the constant time-step size was costly. This motivates us to design adaptive step sizes for synchronization of asymmetric recurrent time-varying neural network.

Moreover, the discretization is related to the concept of sampled-data control. There are a number of papers discussing dynamics of neural networks, using sampled-data control. The papers \citep{Lam,Wu,Zhu} applied the sampled-data control technique towards stabilization of three-layer fully connected feedforward neural networks. In \citet{Chand,Jung,Lee,LiuYu,Rakki}, the authors used sampled-data control strategy for exponential synchronization for the neural networks with Markovian jumping parameters and time varying delays. \citet{Rakki2} discussed state estimation for Markovian jumping fuzzy cellular neural networks with probabilistic time-varying delays with sampled-data.

The purpose of this paper is to give a comprehensive
analysis on out-synchronization of the discrete-time recurrently asymmetrically connected
time-varying neural networks. We propose two schemes of discretizations, named centralized and decentralized
discretization respectively, and present sufficient conditions for the global out-synchronization.
The common step size for every neuron in centralized discretization but
 in decentralized discretization process, the distributed step size for each neuron is used
to guarantee that any two trajectories from different initial values converge together.

\section{Preliminaries and problem formulation\label{Preliminary}}

%{\color{red}[[Define the out-synchronization!!]]}

In this section, we provide the models of asymmetric recurrent neural networks with data-sampling, and some notations. The continuous-time version of the recurrent connected neural networks is described by the following differential equations
\begin{align}
\frac{{\rm d}u_{i}(t)}{{\rm d}t}
=&-\gamma_{i}(t)u_{i}(t)+\sum_{j=1}^{n}a_{ij}(t)g_{j}\big(u_{j}(t)\big)+I_{i}(t),\label{continuous}
\end{align}
where $\gamma_{i}(t)$, $a_{ij}(t)$ and $I_{i}(t)$ are piece-wise continuous and bounded, $\gamma_{i}(t)>0$, and $g_{i}(\cdot)$ satisfies
\begin{align}\label{UpperBound}
0\le\frac{g_{i}(x)-g_{i}(y)}{x-y}\leqslant G_{i}
\end{align}
for all $x\ne y$, where $G_{i}>0$ is a constant and $i=\oneton$.

In the centralized data-sampling strategy, the continuous-time system (\ref{continuous}) is rewritten as
\begin{align}\label{Centralized}
\frac{{\rm d}u_{i}(t)}{{\rm d}t}=-\gamma_{i}(t)u_{i}(t_{k})+\sum_{j=1}^{n}a_{ij}(t)g_{j}\big(u_{j}(t_{k})\big)+I_{i}(t)
\end{align}
for $i=\oneton$. The increasing time sequence $\{t_{k}\}_{k=0}^{+\infty}$ ordered as $0=t_{0}<t_{1}<\cdots<t_{k}<\cdots$ is uniform for all the neuron $i\in\{\oneton\}$ . Each neuron broadcasts its state to its out-neighbours and receives its in-neighbours' states information at time $t_{k}$.

Comparatively, in the decentralized data-sampling strategy, Eq. (\ref{continuous}) is rewritten as the following push-based decentralized system
\begin{align}\label{Push}
\frac{{\rm d}u_{i}(t)}{{\rm d}t}=-\gamma_{i}(t)u_{i}(t_{k}^{i})+\sum_{j=1}^{n}a_{ij}(t)g_{j}\big(u_{j}(t_{k}^{j})\big)+I_{i}(t)
\end{align}
for $i=\oneton$. The increasing time sequence $\{t_{k}^{i}\}_{k=0}^{+\infty}$ order as $0=t_{0}^{i}<t_{1}^{i}<\cdots<t_{k}^{i}<\cdots$ is distributed for the neuron $i\in\{\oneton\}$ . Every neuron $i$ pushes its state information to its out-neighbours at time $t_{k}^{i}$ when it updates its state. It receives its in-neighbours' state information at time $t_{k}^{j}$ when its neighbour neuron $j$ renews it state.

To begin the discussion, we give the following three norms of $\R^{n}$ and recall the definition of out-synchronization proposed in \citet{WuZheng}.

\begin{definition}\label{norm}
Let $\xi_{i}>0~(i=\oneton)$ be a positive constant and we can define three generalized norms as follow
\begin{enumerate}
\renewcommand{\labelenumi}{(\arabic{enumi})}
\item $l_{1}$ norm: $\|x\|_{1,\xi}=\sum\limits_{i=1}^{n}\xi_{i}|x_{i}|$
\item $l_{2}$ norm: $\|x\|_{2,\xi}=\left(\sum\limits_{i=1}^{n}\xi_{i}|x_{i}|^{2}\right)^{1/2}$\\[2pt]
\item $l_{\infty,\xi}$ norm: $\|x\|_{\infty}=\max\limits_{i=\oneton}\xi^{-1}_{i}|x_{i}|$
\end{enumerate}
where $x=[x_{1},\cdots,x_{n}]^{\top}\in\R^{n}$ is a vector.
\end{definition}

\begin{definition}\label{Synchronization}
Consider any two trajectories $u(t)$ and $v(t)$ starting from different initial values $u(0)$ and $v(0)$ of the following system
\begin{align}\label{Synchronization:System}
\frac{{\rm d}x(t)}{{\rm d}t}=f\big(x(t),t\big).
\end{align}
The system \eqref{Synchronization:System} is said to achieve out-synchronization if there exists a controller $c(t)$ for the two trajectories $u(t)$ and $v(t)$ such that
\begin{align*}
\lim_{t\to+\infty}\big\|u(t)-v(t)\big\|=0.
\end{align*}
\end{definition}

Other major notations which will be used throughout this paper are summarized in the following definition.
\begin{definition}\label{mu:nu}
Let $\xi_{i}>0~(i=\oneton)$ be a positive constant and then we define
\begin{align*}
\mu_{1,j}(\xi,t)=&
\gamma_{j}(t)-G_{j}(a_{jj}(t))^{+}-G_{j}\sum_{i\neq j}\frac{\xi_{i}}{\xi_{j}}\big|a_{ij}(t)\big|\\
%--------------------------------------------------------------------------------------------------------------------------
\mu_{2,j}(\xi,t)=&
\gamma_{i}(t)-G_{i}(a_{jj}(t))^{+}\\
&\hspace{1ex}-\frac{1}{2}\sum_{j\neq i}\bigg[G_{j}\big|a_{ij}(t)\big|+G_{i}\frac{\xi_{j}}{\xi_{i}}\big|a_{ji}(t)\big|\bigg]\\
%--------------------------------------------------------------------------------------------------------------------------
\mu_{\infty,j}(\xi,t)=&
\Bigg\{\gamma_{i}(t)-G_{i}(a_{jj}(t))^{+}-G_{i}\sum_{j\neq i}\frac{\xi_{j}}{\xi_{i}}\big|a_{ij}(t)\big|
\end{align*}
and
\begin{align*}
\nu(t)=\max_{j=\oneton}&
\Bigg\{\gamma_{j}(t)-G_{j}(a_{jj}(t))^{-}\Bigg\}\\
\end{align*}
where $
(a)^{+}=\max\{a,0\}$ and $(a)^{-}=\min\{0,a\}$.
\end{definition}
Because of the boundedness of the functions, it can be seen that $\nu(t)$ and $\mu_{1,2,\infty}(\xi,t)$ are bounded for all $t\in[0,+\infty)$. That is, there exist positive constants $\M$ and $\N_{m}$ such that
\begin{align*}
\sup_{t\in[0,+\infty)}\nu(t)\leqslant\M,~\sup_{t\in[0,+\infty)}\mu_{m,j}(\xi,t)\le\N_{m}
\end{align*}
with $m=1,2,\infty$.

\section{Structure-dependent data-sampling principle\label{Structure}}
In this section, we provide several the structure-based data-sampling rules for the next triggering time point at which the neurons renew their states and the control signals.

\subsection{Structure-dependent centralized  data-sampling}
For any neuron $i~(i\in\{\oneton\})$, consider two trajectories $u(t)$ and $v(t)$ of the system \eqref{Centralized} starting from different initial values. Denote $w(t)=[w_{1}(t),\cdots,w_{n}(t)]^{\top}$ with $w_{i}(t)=u_{i}(t)-v_{i}(t)$. And it holds
\begin{align}\label{Discrete:Centralized}
\frac{{\rm d}w_{i}(t)}{{\rm d}t}=-\gamma_{i}(t)w_{i}(t_{k})+\sum_{j=1}^{n}a_{ij}(t)h_{j}(t_{k})
\end{align}
where $h_{i}(t)=g_{i}(u_{i}(t))-g_{i}(v_{i}(t))$ for all $t\in[t_{k},t_{k+1})$, $i=\oneton$ and $k=\zerotoinfty$.

The following theorem gives conditions that guarantee the system \eqref{Centralized} reaches out-synchronization via $l_{1}$- norm.

\begin{theorem}\label{Structure:Centralized}
Let $0<\epsilon_{c}<1$ and $\epsilon_{0}>0$ be constants with $\M\epsilon_{c}\le\epsilon_{0}$ and $\N_{1}\epsilon_{c}\le\epsilon_{0}(2-\epsilon_{c})$. Suppose that there exist $\xi_{i}>0$, $i=\oneton$ such that $\mu_{1,j}(\xi,t)\ge \epsilon_{0}$ for all $j=\oneton$ and $t\ge 0$. Set an increasing time-point sequence $\{t_{k}\}$ as
\begin{align}\label{Structure:Centralized:L1}
&t_{k+1}=\sup_{\tau\geqslant t_{k}}\Bigg\{\tau:\min_{j=\oneton}\int_{t_{k}}^{t}\mu_{1,j}(\xi,s){\rm d}s\le\epsilon_{c},~\forall~t\in(t_{k},\tau]\Bigg\}
\end{align}
%\begin{align}\label{Structure:Centralized:L1}
%\int_{t_{k}}^{t_{k+1}}\bigg[\gamma_{j}(t)-G_{j}a_{jj}^{+}(t)-G_{j}\sum_{i\neq j}\frac{\xi_{i}}{\xi_{j}}\big|a_{ij}(t)\big|\bigg]{\rm d}t
%\geqslant\epsilon_{c}
%\end{align}
$k=1,2,\cdots$. Then the system \eqref{Centralized} reaches out-synchronization.
\end{theorem}

\begin{proof} From the condition $\M\epsilon_{c}\le\epsilon_{0}$, one can see that
\begin{align*}
\int_{t_{k}}^{t}\mu_{1,j}(\xi,s){\rm d}s\ge\epsilon_{0}(t-t_{k}),~\forall~j
\end{align*}
which implies that $t_{k+1}$ exists for all $k$ and $t_{k+1}-t_{k}\le \epsilon_{c}/\epsilon_{0}$. Thus, one can further see
\begin{align}
\int_{t_{k}}^{t}\nu(t){\rm d}t\le \M (t-t_{k})\le \M\epsilon_{c}/\epsilon_{0}\le 1\label{ineq0}
\end{align}
and
\begin{align}
\int_{t_{k}}^{t}\mu_{1,j}(\xi,t){\rm d}t\le \N_{1} (t-t_{k})\le \N_{1}\epsilon_{c}/\epsilon_{0}\le 2-\epsilon_{c}\label{ineq0.1}
\end{align}
for all $j=\oneton$ and $t\in[t_{k},t_{k+1}]$. Furthermore, we have
\begin{align*}
\epsilon_{c}/\epsilon_{0}\ge t_{k+1}-t_{k}\ge \epsilon_{c}/\N_{1}.
\end{align*}

Consider $w_{i}(t)~(i=\oneton)$ for each $t\in[t_{k},t_{k+1}]$, and we have
\begin{align*}
 &\sum_{i=1}^{n}\xi_{i}\big|w_{i}(t)\big|\\
=&\sum_{i=1}^{n}\xi_{i}\left|w_{i}(t_{k})+\int_{t_{k}}^{t}\dot{w}_{i}(s)\,{\rm d}s\right|\\
=&\sum_{i=1}^{n}\left|\xi_{i}w_{i}(t_{k})-\int_{t_{k}}^{t}
\bigg[\gamma_{i}(s)-a_{ii}(s)m_{i}(s)\bigg]{\rm d}s\xi_{i}w_{i}(t_{k})\right.\\
 &\left.+\sum_{j\ne i}^{n}\int_{t_{k}}^{t}\bigg[a_{ij}(s)\frac{\xi_{i}}{\xi_{j}}
 m_{j}(t_{k})\xi_{j}w_{j}(t_{k})\big)\bigg]{\rm d}s\right|
\end{align*}
with
\begin{align*}
m_{j}(t_{k})=\begin{cases}\frac{h_{j}(t_{k})}{w_{j}(t_{k})}&w_{j}(t_{k})\ne 0\\
0&w_{j}(t_{k})=0\end{cases}
\end{align*}
which implies $0\le m_{j}(t_{k})\le G_{j}$ for all $j=\oneton$ and $k=1,2,\cdots$, according to (\ref{UpperBound}). Note
\begin{align*}
(a_{ii}(s))^{-}G_{i}\le a_{ii}(s)m_{i}(t_{k})\le (a_{ii}(s))^{+}G_{i}.
\end{align*}
From (\ref{ineq0}), one can see
\begin{align*}
&\int_{t_{k}}^{t}[\gamma(s)-a_{ii}(s)m_{i}(t_{k})]{\rm d}s\le
\int_{t_{k}}^{t}[\gamma(s)-(a_{ii}(s))^{+}G_{i}]{\rm d}s\\
&\le \int_{t_{k}}^{t} \mu(s)
{\rm d}s\le\M (t_{k+1}-t_{k})\le 1
\end{align*}
which leads
\begin{align}
&0\le 1-\int_{t_{k}}^{t}[\gamma(s)-a_{ii}(s)m_{i}(t_{k})]{\rm d}s\nonumber\\
&\le 1-\int_{t_{k}}^{t}[\gamma(s)-(a_{ii}(s))^{+}G_{i}]ds\label{ineq1}
\end{align}
Then, it follows
\begin{align}
&\sum_{i=1}^{n}\xi_{i}\big|w_{i}(t)\big|
\leqslant
\sum_{j=1}^{n}\Bigg\{\Big|1-\int_{t_{k}}^{t}\big[\gamma_{j}(s)-G_{j}a_{jj}^{+}(s)\big]
{\rm d}s\Big|
\nonumber\\
&+\sum_{i\neq j}\frac{\xi_{i}}{\xi_{j}}\int_{t_{k}}^{t}\big|a_{ij}(s)\big|{\rm d}s\times G_{j}{\rm d}s\Bigg\}\xi_{j}\big|w_{j}(t_{k})\big|\nonumber\\
&=\sum_{i=1}^{n}\bigg|1-\int_{t_{k}}^{t}\mu_{1,j}(\xi,s){\rm d}s\bigg|\xi_{j}|w_{j}(t_{k})|
\label{ineq2}
\end{align}
The last equality holds due to (\ref{ineq1}).
Thus, according to the rule \eqref{Structure:Centralized:L1} and (\ref{ineq0.1}), which implies
\begin{align*}
1-\epsilon_{c}\ge 1-\int_{t_{k}}^{t_{k+1}}\mu_{1,j}(\xi,s){\rm d}s\ge -1+\epsilon_{c},~\forall~j
\end{align*}
since the equality in \eqref{Structure:Centralized:L1} occurs at $t=t_{k+1}$, thus we have
\begin{align*}
\sum_{i=1}^{n}\xi_{i}\big|w_{i}(t_{k+1})\big|\leqslant(1-\epsilon_{c})\sum_{i=1}^{n}\xi_{i}\big|w_{i}(t_{k})\big|
\end{align*}
which implies
\begin{align*}
\lim_{t_{k}\to+\infty}\big\|w(t_{k})\big\|_{1,\xi}=0.
\end{align*}
In addition, for each $t\in(t_{k},t_{k+1})$, from the rule (\ref{Structure:Centralized:L1}) and the condition $\mu(\xi,t)\ge \epsilon_{0}>0$, inequality (\ref{ineq2}) implies that for each $t\in(t_{k},t_{k+1})$, $\|w(t)\|_{1,\xi}\le \|w(t_{k})\|_{1,\xi}$. Hence, it holds
\begin{align*}
\lim_{t\to+\infty}\big\|w(t)\big\|_{1,\xi}=0.
\end{align*}
The out-synchronization of system \eqref{Centralized} is proved.
\end{proof}

The proofs of the following results are analog to Theorem \ref{State:Centralized} but via $l_{2}$ and $l_{\infty}$ norm. Their proofs are similar to that of Theorem \ref{State:Centralized}, which can be found in \citet{Zheng} and so neglected in the present paper.

\begin{proposition}\label{Corollary:Centralized:L2}
Let $0<\epsilon_{c}<1$ and $\epsilon_{0}>0$ be constants with $\M\epsilon_{c}\le\epsilon_{0}$ and $\N_{2}\epsilon_{c}\le\epsilon_{0}(2-\epsilon_{c})$. Suppose that there exist $\xi_{i}>0$, $i=\oneton$ such that $\mu_{2,j}(\xi,t)\ge \epsilon_{0}$ for all $j=\oneton$ and $t\ge 0$. Set an increasing time-point sequence $\{t_{k}\}$ as
\begin{align}\label{Structure:Centralized:L2}
t_{k+1}=\sup_{\tau\geqslant t_{k}}&\Bigg\{\tau:\min_{j=\oneton}\int_{t_{k}}^{t}\mu_{2,j}(\xi,t){\rm d}s\le\epsilon_{c},~\forall~t\in(t_{k},\tau]\Bigg\}
\end{align}
$k=\zerotoinfty$. Then the system \eqref{Centralized} reaches out-synchronization.
\end{proposition}

\begin{proposition}\label{Corollary:Centralized:LInfty}
Let $0<\epsilon_{c}<1$ and $\epsilon_{0}>0$ be constants with $\M\epsilon_{c}\le\epsilon_{0}$ and $N_{\infty}\epsilon_{c}\le\epsilon_{0}(2-\epsilon_{c})$. Suppose that there exist $\xi_{i}>0$, $i=\oneton$ such that $\mu_{\infty,j}(\xi,t)\ge \epsilon_{0}$ for all $j=\oneton$ and $t\ge 0$. Set an increasing time-point sequence $\{t_{k}\}$ as
\begin{align}\label{Structure:Centralized:LInfty}
t_{k+1}=\max_{\tau\geqslant t_{k}}\Bigg\{\tau&:\min_{j=\oneton}\int_{t_{k}}^{t}\mu_{\infty,j}(\xi,t){\rm d}s\le\epsilon_{c},~\forall~t\in(t_{k},\tau]\Bigg\}
\end{align}
$k=\zerotoinfty$. Then the system \eqref{Centralized} reaches out-synchronization.
\end{proposition}
\begin{remark}
From the proof, one can see  $t_{k+1}-t_{k}\ge \epsilon_{c}/\N_{m}$, which excludes the Zeno behaviours for the rules (\ref{Structure:Centralized:L1},\ref{Structure:Centralized:L2},\ref{Structure:Centralized:LInfty}).
\end{remark}
To explain the independence of the results via three norms, we give out the following example.
Denote
\begin{align*}
&\mathcal{L}_{1}=
\begin{bmatrix}
\gamma_{1}(t)-G_{1}a_{11}^{+}(t)-G_{1}\dfrac{\xi_{2}}{\xi_{1}}\big|a_{21}(t)\big|\\[10pt]
\gamma_{2}(t)-G_{2}a_{22}^{+}(t)-G_{2}\dfrac{\xi_{1}}{\xi_{2}}\big|a_{12}(t)\big|
\end{bmatrix}\\
&\mathcal{L}_{2}=
\begin{bmatrix}
\gamma_{1}(t)-G_{1}a_{11}^{+}(t)-\dfrac{1}{2}G_{1}\dfrac{\xi'_{2}}{\xi'_{1}}\big|a_{21}(t)\big|-\dfrac{1}{2}G_{2}\big|a_{12}(t)\big|\\[10pt]
\gamma_{2}(t)-G_{2}a_{22}^{+}(t)-\dfrac{1}{2}G_{2}\dfrac{\xi'_{1}}{\xi'_{2}}\big|a_{12}(t)\big|-\dfrac{1}{2}G_{1}\big|a_{21}(t)\big|
\end{bmatrix}.
\end{align*}
Let
%\begin{align*}
%\begin{bmatrix}G_{1},~G_{2}\end{bmatrix}=\begin{bmatrix}1.0017,~0.9984\end{bmatrix}
%\end{align*}
$\begin{bmatrix}G_{1},G_{2}\end{bmatrix}=\begin{bmatrix}1.0017,0.9984\end{bmatrix}$ with
\begin{align*}
\gamma(t)=\begin{bmatrix} 2.1048\\ 0.9234 \end{bmatrix}\hspace{3ex}
A(t)&=\begin{bmatrix} ~~~1.0235 & ~~~0.2538 \\ ~~~0.5014 & -0.1526\end{bmatrix}
\end{align*}
when $t\in [t_{k},t_{k+1})$ and
\begin{align*}
\gamma(t)=\begin{bmatrix} 2.1048\\ 0.9234 \end{bmatrix}\hspace{3ex}
A(t)&=\begin{bmatrix}-0.3253 & ~~~0.4384 \\ -2.0341 & -0.1526\end{bmatrix}
\end{align*}
when $t\in [t_{k+1},t_{k+2})$.

In the first time interval $[t_{k},t_{k+1})$, we have found that when
\begin{align*}
\begin{bmatrix} \xi_{1}\\\xi_{2} \end{bmatrix}=\begin{bmatrix} 0.8902\\0.3562 \end{bmatrix}\hspace{5ex}
\begin{bmatrix} \xi'_{1}\\\xi'_{2} \end{bmatrix}=\begin{bmatrix} 0.3479\\0.7727 \end{bmatrix}
\end{align*}
it holds
\begin{align*}
\mathcal{L}_{1}=\begin{bmatrix} 0.8786\\ 0.2901 \end{bmatrix}>\mathbf{0}\hspace{3ex}
\mathcal{L}_{2}=\begin{bmatrix} 0.3951\\ 0.6210 \end{bmatrix}>\mathbf{0}
\end{align*}
where $\mathbf{0}=[0,0]^{\top}$.
In the second time interval $[t_{k+1},t_{k+2})$, we can find that when
\begin{align*}
\begin{bmatrix} \xi_{1}\\\xi_{2} \end{bmatrix}=\begin{bmatrix} 0.7182\\0.3570 \end{bmatrix}£¬
\end{align*}
it follows
\begin{align*}
\mathcal{L}_{1}=\begin{bmatrix}1.0920\\0.0429\end{bmatrix}>\mathbf{0}.
\end{align*}
However, to maintain $\mathcal{L}_{2}>\mathbf{0}$, we have to solve the following inequalities
\begin{align*}
\begin{cases}
\gamma_{1}(t)-G_{1}a_{11}^{+}(t)-\dfrac{1}{2}G_{1}\dfrac{\xi'_{2}}{\xi'_{1}}\big|a_{21}(t)\big|-\dfrac{1}{2}G_{2}\big|a_{12}(t)\big|>0\\[10pt]
\gamma_{2}(t)-G_{2}a_{22}^{+}(t)-\dfrac{1}{2}G_{2}\dfrac{\xi'_{1}}{\xi'_{2}}\big|a_{12}(t)\big|-\dfrac{1}{2}G_{1}\big|a_{21}(t)\big|>0\\
\end{cases}
\end{align*}
that is
\begin{align*}
\begin{cases}
1.8860-1.0189\,\dfrac{\xi'_{2}}{\xi'_{1}}>0\\[10pt]
-0.0954-0.2188\,\dfrac{\xi'_{1}}{\xi'_{2}}>0\\
\end{cases}.
\end{align*}
One can see that there is no such solution of $\xi'_{1}$ and $\xi'_{2}$.

Hence, the conditions \eqref{Structure:Centralized:L1}, \eqref{Structure:Centralized:L2} are uncorrelated. By using the similar method, we can obtain that the three conditions are pairwise independent. Therefore, we can assert that the results via three norms are independent.

%{\color{red}[[A remark  that compares the conditions are needed!]]}

\subsection{Structure-dependent push-based decentralized data-sampling}

For each neuron $i\in\{\oneton\}$, consider two trajectories $u(t)$ and $v(t)$ of system \eqref{Push} starting from different initial values.
Denote $w(t)=[w_{1}(t),\cdots,w_{n}(t)]^{\top}$ with $w_{i}(t)=u_{i}(t)-v_{i}(t)$. It follows
\begin{align}\label{Discrete:Push}
\frac{{\rm d}w_{i}(t)}{{\rm d}t}=-\gamma_{i}(t)w_{i}(t_{k}^{i})+\sum_{j=1}^{n}a_{ij}(t)h_{j}(t_{k}^{j})
\end{align}
where $h_{j}(t)=g_{j}(u_{j}(t))-g_{j}(v_{j}(t))$ for all $t\in[t_{k}^{j},t_{k+1}^{j})$, $i=\oneton$ and $k=\zerotoinfty$.

The following theorem and propositions give conditions that guarantee the convergence of system \eqref{Discrete:Push} via three generalized norms ($l_{1}$, $l_{2}$ and $l_{\infty}$).

\begin{theorem}\label{Structure:Push}
Let $0<\epsilon_{d}<1$ and $\epsilon_{0}>0$ be constants with $\M\epsilon_{d}\le\epsilon_{0}$ and $\N_{1}\epsilon_{d}\le\epsilon_{0}(2-\epsilon_{d})$. Suppose that there exist $\xi_{i}>0~(i=\oneton)$ such that $\mu_{1,j}(\xi,t)\ge\epsilon_{0}$ for all $j=\oneton$ and $t\ge 0$. Set $\{t^{j}_{k}\}_{k=0}^{+\infty}$ as the triggering time points as
\begin{align}\label{Structure:Push:L1}
t_{k+1}^{j}=\sup_{\tau\geqslant t_{k}^{j}}\Bigg\{\tau&:\int_{t_{k}^{j}}^{t}\Big[\gamma_{j}(s)-G_{j}a_{jj}^{+}(s)\Big]{\rm d}s
-G_{j}\sum_{i\neq j}\frac{\xi_{i}}{\xi_{j}}\notag\\
&\times\int_{t_{k}^{i}}^{\tau}\big|a_{ij}(s)\big|\,{\rm d}s\geqslant\epsilon_{d},~\forall~t\in(t_{k}^{j},\tau]\Bigg\}
\end{align}
for $j=\oneton$ and $k=\zerotoinfty$. Then the system \eqref{Push} reaches out-synchronization.
\end{theorem}
\begin{proof} For each $t\ge 0$, let $k_{j}(t)=\max\{k:~t^{j}_{k}\le t\}$. Similar to the arguments up to (\ref{ineq2}) in the proof of Theorem \ref{Structure:Centralized}, one can derive the following inequality immediately:
\begin{align}
\sum_{i=1}^{n}\xi_{i}|w_{i}(t)|\le \sum_{i=1}^{n}\bigg|1-\int_{t^{j}_{k_{j}(t)}}^{t}\mu_{1,j}(\xi,s){\rm d}s\bigg|\xi_{j}|w_{j}(t^{j}_{k_{j}(t)})|\label{ineq3}
\end{align}
From the arguments of (\ref{ineq0.1}), one can conclude
\begin{align}
1\ge 1-\int_{t^{j}_{k_{j}(t)}}^{t}\mu_{1,j}(\xi,s){\rm d}s\ge -1+\epsilon_{c},~\forall~j.\label{ineq3.1}
\end{align}
in an analog way.

Let $t_{k+1}$ be an increasing sequence such that $t_{0}=0$ and $t^{k+1}-t^{k}=2\epsilon_{d}/\epsilon_{0}$, which implies that for each neuron $j$, equality in the rule (\ref{Structure:Push:L1}) occurs at least once. Thus, we have

Consider $w_{i}(t)$ for any neuron $i$ at triggering time $t_{k+1}^{i}$ where $i=\oneton$, and we have
\begin{align*}
 &\sum_{i=1}^{n}\xi_{i}\big|w_{i}(t_{k+1}^{i})\big|\\
=&\sum_{i=1}^{n}\sign\big(w_{i}(t_{k+1}^{i})\big)\xi_{i}\bigg[w_{i}(t_{k}^{i})+\int_{t_{k}^{i}}^{t_{k+1}^{i}}\dot{w}_{i}(s)\,{\rm d}s\bigg]\\
=&\sum_{i=1}^{n}\sign\big(w_{i}(t_{k+1}^{i})\big)\xi_{i}w_{i}(t_{k}^{i})+\sum_{i=1}^{n}\sign\big(w_{i}(t_{k+1}^{i})\big)\xi_{i}\\
 &\times\int_{t_{k}^{i}}^{t_{k+1}^{i}}\bigg[-\gamma_{i}(s)w_{i}(t_{k}^{i})+\sum_{j=1}^{n}a_{ij}(s)h_{j}\big(w_{j}(t_{k}^{j})\big)\bigg]{\rm d}s\\
%----------------------------------------------------------------------------------------------------------------------------------------------
=&\sum_{i=1}^{n}\sign\big(w_{i}(t_{k+1}^{i})\big)\xi_{i}w_{i}(t_{k}^{i})\\
 &-\sum_{i=1}^{n}\sign\big(w_{i}(t_{k+1}^{i})\big)\xi_{i}w_{i}(t_{k}^{i})\int_{t_{k}^{i}}^{t_{k+1}^{i}}\gamma_{i}(s)\,{\rm d}s\\
 &+\sum_{i=1}^{n}\sign\big(w_{i}(t_{k+1}^{i})\big)\xi_{i}h_{i}\big(w_{i}(t_{k}^{i})\big)\int_{t_{k}^{i}}^{t_{k+1}^{i}}a_{ii}^{+}(s)\,{\rm d}s\\
 &+\sum_{i\neq j}\sign\big(w_{i}(t_{k+1}^{i})\big)\xi_{i}\sum_{j=1}^{n}h_{j}\big(w_{j}(t_{k}^{j})\big)\int_{t_{k}^{i}}^{t_{k+1}^{i}}a_{ij}(s)\,{\rm d}s.
%----------------------------------------------------------------------------------------------------------------------------------------------
\end{align*}
By  the inequality  \eqref{UpperBound}, it holds
\begin{align*}
 \sum_{i=1}^{n}\xi_{i}\big|w_{i}(t_{k+1}^{i})\big|
\leqslant
&\sum_{j=1}^{n}\Bigg\{1-\int_{t_{k}^{j}}^{t_{k+1}^{j}}\Big[\gamma_{j}(s)-G_{j}a_{jj}^{+}(s)\Big]{\rm d}s\\
&+G_{j}\sum_{i\neq j}\frac{\xi_{i}}{\xi_{j}}\int_{t_{k}^{i}}^{t_{k+1}^{i}}\big|a_{ij}(s)\big|\,{\rm d}s\Bigg\}
  \xi_{j}\big|w_{j}(t_{k}^{j})\big|\\
\leqslant
&~(1-\epsilon_{d})\sum_{j=1}^{n}\xi_{j}\big|w_{j}(t_{k}^{j})\big|
\end{align*}
Based on the triggering rule \eqref{Structure:Push:L1}, we can obtain
\begin{align*}
\sum_{i=1}^{n}\xi_{i}\big|w_{i}(t_{k+1}^{i})\big|\leqslant(1-\epsilon_{d})\sum_{i=1}^{n}\xi_{i}\big|w_{i}(t_{k}^{i})\big|
\end{align*}
which means
\begin{align*}
\lim_{t_{k}^{i}\to+\infty}\big\|w(t_{k}^{i})\big\|_{1}=0.
\end{align*}
For any time $t\in(t_{k}^{i},t_{k+1}^{i}]$, the state $w_{i}(t)$ becomes
\begin{align*}
&\sum_{i=1}^{n}\xi_{i}\big|w_{i}(t)\big|\\
\leqslant
&\sum_{j=1}^{n}\Bigg\{1-\int_{t_{k}^{j}}^{t}\Big[\gamma_{j}(s)-G_{j}a_{jj}^{+}(s)\Big]{\rm d}s\\
&\hspace{5ex}+G_{j}\sum_{i\neq j}\frac{\xi_{i}}{\xi_{j}}\int_{t_{k}^{i}}^{t}\big|a_{ij}(s)\big|\,{\rm d}s\Bigg\}\,\xi_{j}\big|w_{j}(t_{k}^{j})\big|\\
\leqslant
&\sum_{j=1}^{n}\Bigg\{1-\int_{t_{k}^{j}}^{t}\Big[\gamma_{j}(t)-G_{j}a_{jj}^{+}(t)\Big]{\rm d}t+G_{j}\sum_{i\neq j}\frac{\xi_{i}}{\xi_{j}}\\
&\hspace{5ex}\times\bigg[\int_{t_{k}^{i}}^{t}\big|a_{ij}(s)\big|\,{\rm d}s+\int_{t}^{\tau}\big|a_{ij}(s)\big|\,{\rm d}s\bigg]\Bigg\}\,\xi_{j}\big|w_{j}(t_{k}^{j})\big|\\
\leqslant&~(1-\epsilon_{d})\sum_{i=1}^{n}\xi_{i}\big|w_{i}(t_{k}^{i})\big|
\end{align*}
where $t_{k+1}^{i}\geqslant\tau>t>t_{k}^{i}$. Thus,
\begin{align*}
\big\|w(t_{k+1}^{i})\big\|_{1}\leqslant\big\|w(t)\big\|_{1}\leqslant(1-\epsilon_{d})\big\|w(t_{k}^{i})\big\|_{1}
\end{align*}
for any $t\in(t_{k}^{i},t_{k+1}^{i}]$ and $i=\oneton$, which implies
\begin{align*}
\lim_{t\to+\infty}\big\|w(t)\big\|_{1}\leqslant\lim_{t_{k}^{i}\to+\infty}\big\|w(t_{k}^{i})\big\|_{1}=0.
\end{align*}
The proof for the out-synchronization of system \eqref{Push} is completed.
\end{proof}

%By similar approach, we can give following propositions via $l_{2}$ and $l_{\infty}$ norm without detailed proof.
%The following results are analog to Theorem \ref{Structure:Push} but via $l_{2}$- and $l_{\infty}$- norm. Their proofs are similar to that of Theorem \ref{Structure:Push}, which can be found in \cite{Zheng}.

\begin{proposition}\label{Corollary:Push:L2}
Let $0<\epsilon_{d}<1$ be a constant and $\xi_{i}>0~(i=\oneton)$. Set $\{t_{k}^{i}\}_{k=0}^{+\infty}$ as the time points such that
\begin{align}\label{Structure:Push:L2}
%t_{k+1}^{i}=\max_{\tau\geqslant t_{k}^{i}}\Bigg\{&\tau:\int_{t_{k}^{i}}^{t}\bigg[\gamma_{i}(s)-G_{i}a_{ii}^{+}(s)-\frac{1}{2}\sum_{j\neq i}G_{i}\frac{\xi_{j}}{\xi_{i}}\notag\\
%&\times\big|a_{ji}(t)\big|\bigg]{\rm d}s-\frac{1}{2}\sum_{j\neq i}G_{j}\int_{t_{k}^{j}}^{\tau}\big|a_{ij}(s)\big|\,{\rm d}t\geqslant\epsilon_{d},\notag\\
%&~\forall~t\in(t_{k}^{i},\tau]\Bigg\}
%------------------------------------------------------------------------------------------------------------------------------------------------------
t_{k+1}^{i}=\max_{\tau\geqslant t_{k}^{i}}&\Bigg\{\tau:\int_{t_{k}^{i}}^{t}\bigg[\gamma_{i}(s)-G_{i}a_{ii}^{+}(s)-\frac{1}{2}\sum_{j\neq i}G_{j}\big|a_{ij}(s)\big|\bigg]{\rm d}s\notag\\
&-\frac{1}{2}G_{i}\sum_{j\neq i}\frac{\xi_{j}}{\xi_{i}}\int_{t_{k}^{j}}^{\tau}\big|a_{ji}(t)\big|\,{\rm d}t\geqslant\epsilon_{d},~\forall~t\in(t_{k}^{i},\tau]\Bigg\}
\end{align}
for $i=\oneton$ and $k=\zerotoinfty$. Then the system \eqref{Push} reaches out-synchronization.
\end{proposition}

\begin{proposition}\label{Corollary:Push:LInfty}
Let $0<\epsilon_{d}<1$ be a constant and $\xi_{i}>0~(i=\oneton)$. Set $\{t_{k}^{i}\}_{k=0}^{+\infty}$ as the time points such that
\begin{align}\label{Structure:Push:LInfty}
t_{k+1}^{i}=\max_{\tau\geqslant t_{k}^{i}}\Bigg\{\tau&:\int_{t_{k}^{i}}^{t}\Big[\gamma_{i}(s)-G_{i}a_{ii}^{+}(s)\Big]{\rm d}s
-G_{i}\sum_{j\neq i}\frac{\xi_{j}}{\xi_{i}}\notag\\
&\times\int_{t_{k}^{j}}^{\tau}\big|a_{ij}(s)\big|{\rm d}s\geqslant\epsilon_{d},~\forall~t\in(t_{k}^{i},\tau]\Bigg\}
\end{align}
for $i=\oneton$ and $k=\zerotoinfty$. Then the system \eqref{Push} reaches out-synchronization.
\end{proposition}

\begin{remark}
In centralized data-sampling rules \eqref{Structure:Centralized:L1}, \eqref{Structure:Centralized:L2} and \eqref{Structure:Centralized:LInfty}, the design of time sequence $\{t_{k}\}_{k=0}^{+\infty}$ can guarantee that the Zeno behavior is excluded, because there exists a common positive lower bound
$\frac{\epsilon_{c}}{\M_{m}}$ $(m=1,2 \text{~or~} \infty)$ for the inter-event time $t_{k+1}-t_{k}$ for all the neurons $i=\oneton$ and $k=\zerotoinfty$.

Besides, in push-based decentralized updating rules \eqref{Structure:Push:L1}, \eqref{Structure:Push:L2} and \eqref{Structure:Push:LInfty}, the controlling time sequence $\{t_{k}^{i}\}_{k=0}^{+\infty}$ for $i=\oneton$ can also ensure the exclusion of Zeno behavior, because each inter-event time $t_{k+1}^{i}-t_{k}^{i}$ can be lower bounded by $\frac{\epsilon_{d}}{\M_{m}}~(m=1,2 \text{~or~} \infty)$ for all the neurons $i=\oneton$ and $k=\zerotoinfty$.
\end{remark}

\section{State-dependent  data-sampling principle\label{State}}
In this section, we establish a group of state-dependent data-sampling rules by predicting the next triggering time point at which neurons should broadcast their state information and update their control signals.

\subsection{State-dependent centralized data-sampling}
Consider the system \eqref{Discrete:Centralized} and define the state measurement error vector $e(t)=[e_{1}(t),\cdots,e_{n}(t)]^{\top}$ as
\begin{align*}
e_{i}(t)=w_{i}(t_{k})-w_{i}(t)
\end{align*}
where $t\in[t_{k},t_{k+1})$, $i=\oneton$ and $k=\zerotoinfty$. The centralized updating rule relied on neurons' states is given as follow.

\begin{theorem}\label{State:Centralized}
Let $\Phi(t)$ be a positive decreasing continuous function on $[0,+\infty)$ with $\Phi(0)>0$. Set $t_{k+1}$ as the triggering time point such that
\begin{align}\label{State:Centralized:L1}
t_{k+1}=\max_{\tau\geqslant t_{k}}\Big\{\tau:\big\|e(t)\big\|_{1}\leqslant\Phi(t),~\forall~t\in[t_{k},\tau)\Big\}.
\end{align}
for all $k=\zerotoinfty$. If $\mu_{1}(t)\geqslant\varepsilon_{c}$ for some $\varepsilon_{c}>0$ and
\begin{align*}
\lim_{t\to+\infty}\Phi(t)=0,
\end{align*}
then the system \eqref{Centralized} reaches out-synchronization.
\end{theorem}

\begin{proof}
Consider $w_{i}(t)$ for any neuron $i~(i=\oneton)$ and $\xi_{i}>0~(i=\oneton)$
\begin{align*}
 &\frac{{\rm d}\big\|w(t)\big\|_{1}}{{\rm d}t}\\
=&\sum_{i=1}^{n}\xi_{i}\sign\big(w_{i}(t)\big)\frac{{\rm d}w_{i}(t)}{{\rm d}t}\\
=&\sum_{i=1}^{n}\xi_{i}\sign\big(w_{i}(t)\big)\bigg[-\gamma_{i}(t)w_{i}(t_{k})+\sum_{j=1}^{n}a_{ij}(t)h_{j}\big(w_{j}(t_{k})\big)\bigg]\\
%----------------------------------------------------------------------------------------------------------------------------------------------
=&\sum_{i=1}^{n}\xi_{i}\sign\big(w_{i}(t)\big)\bigg[-\gamma_{i}(t)w_{i}(t)+\sum_{j=1}^{n}a_{ij}(t)h_{j}\big(w_{j}(t)\big)\bigg]\\
 &+\sum_{i=1}^{n}\xi_{i}\sign\big(w_{i}(t)\big)\bigg[-\gamma_{i}(t)\Big(w_{i}(t_{k})-w_{i}(t)\Big)\bigg]\\
 &+\sum_{i=1}^{n}\xi_{i}\sign\big(w_{i}(t)\big)\sum_{j=1}^{n}a_{ij}(t)\bigg[h_{j}\big(w_{j}(t_{k})\big)-h_{j}\big(w_{j}(t)\big)\bigg].
%----------------------------------------------------------------------------------------------------------------------------------------------
%=&~\sum_{i=1}^{n}\sign\big(w_{i}(t_{k+1})\big)\xi_{i}\Big[1-(t_{k+1}-t_{k})\gamma_{i}(t)\Big]w_{i}(t_{k})\\
% &+\sum_{j=1}^{n}\sign\big(w_{j}(t_{k+1})\big)\xi_{j}(t_{k+1}-t_{k})a_{jj}(t)h_{j}(t_{k})w_{j}(t_{k})\\
% &+\sum_{i\neq j}\sign\big(w_{i}(t_{k+1})\big)\xi_{i}(t_{k+1}-t_{k})\sum_{j=1}^{n}a_{ij}(t)h_{j}(t_{k})w_{j}(t_{k})
\end{align*}
%where
%\begin{align}\label{H:Continuous}
%h_{i}\big(w_{i}(t)\big)=g_{i}\big(u_{i}(t)\big)-g_{i}\big(v_{i}(t)\big)
%h_{i}\big(w_{i}(t_{k})\big)=g_{i}\big(u_{i}(t_{k})\big)-g_{i}\big(v_{i}(t_{k})\big)
%\end{align}
%for $t\in[t_{k},t_{k+1})$ and $i=\oneton$.
By the inequality \eqref{UpperBound} and the triggering rule \eqref{State:Centralized:L1}, it holds
\begin{align*}
&\frac{{\rm d}\big\|w(t)\big\|_{1}}{{\rm d}t}\\
 \leqslant
&-\sum_{i=1}^{n}\xi_{i}\gamma_{i}(t)\big|w_{i}(t)\big|
 +\sum_{i=1}^{n}\xi_{i}\gamma_{i}(t)\big|e_{i}(t)\big|\\
&+\sum_{i=1}^{n}\xi_{i}a_{ii}^{+}(t)G_{i}\big|w_{i}(t)\big|
 +\sum_{i=1}^{n}\xi_{i}a_{ii}^{+}(t)G_{i}\big|e_{i}(t)\big|\\
&+\sum_{j=1}^{n}\sum_{i\neq j}\xi_{i}\big|a_{ij}(t)\big|G_{j}\big|w_{j}(t)\big|
 +\sum_{j=1}^{n}\sum_{i\neq j}\xi_{i}\big|a_{ij}(t)\big|G_{j}\big|e_{j}(t)\big|\\
%----------------------------------------------------------------------------------------------------------------------------------------------
 \leqslant
&-\sum_{j=1}^{n}\bigg[\gamma_{j}(t)-G_{j}a_{jj}^{+}(t)-G_{j}\sum_{i\neq j}\frac{\xi_{i}}{\xi_{j}}\big|a_{ij}(t)\big|\bigg]\xi_{j}\big|w_{j}(t)\big|\\
&+\sum_{j=1}^{n}\bigg[\gamma_{j}(t)+G_{j}a_{jj}^{+}(t)+G_{j}\sum_{i\neq j}\frac{\xi_{i}}{\xi_{j}}\big|a_{ij}(t)\big|\bigg]\xi_{j}\big|e_{j}(t)\big|
%----------------------------------------------------------------------------------------------------------------------------------------------
% \leqslant
%&-\sum_{i=1}^{n}\xi_{i}(1-\alpha)\gamma_{i}(t)\big|w_{i}(t)\big|
% +\sum_{j=1}^{n}\xi_{j}(1+\alpha)a_{jj}^{+}(t)G_{j}\big|w_{j}(t)\big|\\
%&+\sum_{j=1}^{n}\sum_{i\neq j}\xi_{i}(1+\alpha)\big|a_{ij}(t)\big|G_{j}\big|w_{j}(t)\big|
%----------------------------------------------------------------------------------------------------------------------------------------------
\end{align*}
which implies
\begin{align*}
\frac{{\rm d}\big\|w(t)\big\|_{1}}{{\rm d}t}
%\leqslant
%&-(1+\alpha)\sum_{j=1}^{n}\xi_{j}\big|w_{j}(t)\big|\\
%&\times\bigg[\beta\gamma_{j}(t)-G_{j}a_{jj}^{+}(t)-\sum_{i\neq j}\frac{\xi_{i}}{\xi_{j}}\big|a_{ij}(t)\big|G_{j}\bigg]\\
%---------------------------------------------------------------------------------------------------------------------------------------------
\leqslant&-\mu_{1}(t)\sum_{j=1}^{n}\xi_{j}\big|w_{j}(t)\big|+\M_{1}\sum_{j=1}^{n}\xi_{j}\big|e_{j}(t)\big|\\
\leqslant&-\mu_{1}(t)\big\|w(t)\big\|_{1}+\M_{1}\Phi(t)
%\leqslant&-\Big[\mu_{1}(t)-\alpha\nu_{1}(t)\Big]\big\|w(t)\big\|_{1},
\end{align*}
By the classical Gronwell inequality, we have
\begin{align*}
\big\|w(t)\big\|_{1}
\leqslant&~\big\|w(t_{0})\big\|_{1}{\rm e}^{-\sigma(t,t_{0})}+\M_{1}\int_{t_{0}}^{t}{\rm e}^{-\sigma(t,s)}\Phi(s)\,{\rm d}s\\
%---------------------------------------------------------------------------------------------------------------------------------------------
=&~{\rm e}^{-\sigma(t,t_{0})}\bigg[\big\|w(t_{0})\big\|_{1}+\M_{1}\int_{t_{0}}^{t}{\rm e}^{\sigma(s,t_{0})}\Phi(s)\,{\rm d}s\bigg]
\end{align*}
with
\begin{align*}
\sigma(t,s)=\int_{s}^{t}\mu_{1}(\tau)\,{\rm d}\tau
\end{align*}
for $s\in[t_{k},t]$, $t\in[t_{k},t_{k+1})$. By using the L'Hospital rule, it follows
%\begin{align*}
% &\lim_{t\to+\infty}\int_{t_{0}}^{t}{\rm e}^{-\sigma(t,s)}\Phi(s)\,{\rm d}s\\=&
%  \lim_{t\to+\infty}\frac{1}{{\rm e}^{\sigma(t,t_{0})}}\int_{t_{0}}^{t}{\rm e}^{\sigma(s,t_{0})}\Phi(s)\,{\rm d}s
%=&\lim_{t\to+\infty}\frac{\Phi(t)\,{\rm e}^{\sigma(t,t_{0})}}{\mu_{1}(t)\,{\rm e}^{\sigma(t,t_{0})}}=0
%\end{align*}
%which means
%\begin{align*}
%\lim_{t\to+\infty}\big\|w(t)\big\|_{1}\leqslant
%\lim_{t\to+\infty}\M_{1}\int_{t_{0}}^{t}{\rm e}^{-\sigma(t,s)}\Phi(s)\,{\rm d}s=0.
%\lim_{t\to+\infty}\frac{\M_{1}}{{\rm e}^{\sigma(t,t_{0})}}\int_{t_{0}}^{t}{\rm e}^{\sigma(s,t_{0})}\Phi(s)\,{\rm d}s=0.
%\end{align*}
\begin{align*}
 \lim_{t\to+\infty}\big\|w(t)\big\|_{1}
&\leqslant\lim_{t\to+\infty}\frac{\M_{1}}{{\rm e}^{\sigma(t,0)}}\int_{0}^{t}{\rm e}^{\sigma(s,0)}\Phi(s)\,{\rm d}s\\
&=\lim_{t\to+\infty}\M_{1}\frac{\Phi(t)\,{\rm e}^{\sigma(t,0)}}{\mu_{1}(t)\,{\rm e}^{\sigma(t,0)}}\\
&\leqslant\lim_{t\to+\infty}\frac{\M_{1}}{\varepsilon_{c}}\Phi(t)\\
&=0,
\end{align*}
where $t_{0}=0$, which implies that $\|w(t)\|_{1}$ converges to $0$ by the sampling time sequence $\{t_{k}\}_{k=0}^{+\infty}$.
Therefore, the system \eqref{Centralized} reaches out-synchronization and this completes the proof of Theorem \ref{State:Centralized}.
\end{proof}

\begin{proposition}
Let $\Phi(t)$ be a positive decreasing continuous function on $[0,+\infty)$ with $\Phi(0)>0$. Set $t_{k+1}$ as the triggering time point such that
\begin{align*}
t_{k+1}=\max_{\tau\geqslant t_{k}}\Big\{\tau:\big\|e(t)\big\|_{2}\leqslant\Phi(t),~\forall~t\in[t_{k},\tau)\Big\}
\end{align*}
for all $k=\zerotoinfty$. If $\mu_{2}(t)\geqslant\varepsilon_{c}$ for some $\varepsilon_{c}>0$ and
\begin{align*}
\lim_{t\to+\infty}\Phi(t)=0,
\end{align*}
then the system \eqref{Centralized} reaches out-synchronization.
\end{proposition}

\begin{proposition}
Let $\Phi(t)$ be a positive decreasing continuous function on $[0,+\infty)$ with $\Phi(0)>0$. Set $t_{k+1}$ as the triggering time point such that
\begin{align*}
t_{k+1}=\max_{\tau\geqslant t_{k}}\Big\{\tau:\big\|e(t)\big\|_{\infty}\leqslant\Phi(t),~\forall~t\in[t_{k},\tau)\Big\}
\end{align*}
for all $k=\zerotoinfty$. If $\mu_{\infty}(t)\geqslant\varepsilon_{c}$ for some $\varepsilon_{c}>0$ and
\begin{align*}
\lim_{t\to+\infty}\Phi(t)=0,
\end{align*}
then the system \eqref{Centralized} reaches out-synchronization.
\end{proposition}

\subsection{State-dependent push-based decentralized data-sampling}
For the system \eqref{Discrete:Push}, we define the state measurement error vector $e(t)=[e_{1}(t),\cdots,e_{n}(t)]^{\top}$ as
\begin{align*}
e_{i}(t)=w_{i}(t_{k}^{i})-w_{i}(t)
\end{align*}
where $t\in[t_{k}^{i},t_{k+1}^{i})$, $i=\oneton$ and $k=\zerotoinfty$. The push-based decentralized updating rule is given as follow.

\begin{theorem}\label{State:Push}
Let $\Psi(t)=[\Psi_{1}(t),\cdots,\Psi_{n}(t)]^{\top}$ be a vector of positive decreasing continuous functions with $\Psi(0)>0$, that is, $\Psi_{i}(t)$ is a positive decreasing continuous function on $[0,+\infty)$ and $\Psi_{i}(0)>0$ for $i=\oneton$. Set $t_{k+1}^{i}$ as the triggering time point such that
\begin{align}\label{State:Push:L1}
t_{k+1}^{i}
=\max_{\tau\geqslant t_{k}^{i}}\Big\{\tau:\big|e_{i}(t)\big|\leqslant\Psi_{i}(t),~\forall~t\in[t_{k}^{i},\tau)\Big\}
\end{align}
for $i=\oneton$ and all $k=\zerotoinfty$. If $\mu_{1}(t)\geqslant\varepsilon_{d}$ for some $\varepsilon_{d}>0$ and
\begin{align*}
\lim_{t\to+\infty}\Psi_{i}(t)=0
\end{align*}
for $i=\oneton$, then the system \eqref{Push} reaches out-synchronization.
\end{theorem}

\begin{proof}
Consider the $l_{1}$-norm of the state $w_{i}(t)~(i=\oneton)$ and $\xi_{i}>0~(i=\oneton)$
\begin{align*}
 &\frac{{\rm d}\big\|w(t)\big\|_{1}}{{\rm d}t}\\
=&\sum_{i=1}^{n}\xi_{i}\sign\big(w_{i}(t)\big)\frac{{\rm d}w_{i}(t)}{{\rm d}t}\\
=&\sum_{i=1}^{n}\xi_{i}\sign\big(w_{i}(t)\big)\bigg[-\gamma_{i}(t)w_{i}(t_{k}^{i})+\sum_{j=1}^{n}a_{ij}(t)h_{j}\big(w_{j}(t_{k}^{j})\big)\bigg]\\
%----------------------------------------------------------------------------------------------------------------------------------------------
=&\sum_{i=1}^{n}\xi_{i}\sign\big(w_{i}(t)\big)\bigg[-\gamma_{i}(t)w_{i}(t)+\sum_{j=1}^{n}a_{ij}(t)h_{j}\big(w_{j}(t)\big)\bigg]\\
 &+\sum_{i=1}^{n}\xi_{i}\sign\big(w_{i}(t)\big)\bigg[-\gamma_{i}(t)\Big(w_{i}(t_{k}^{i})-w_{i}(t)\Big)\bigg]\\
 &+\sum_{i=1}^{n}\xi_{i}\sign\big(w_{i}(t)\big)\sum_{j=1}^{n}a_{ij}(t)\bigg[h_{j}\big(w_{j}(t_{k}^{j})\big)-h_{j}\big(w_{j}(t)\big)\bigg].
%----------------------------------------------------------------------------------------------------------------------------------------------
\end{align*}
By the inequality \eqref{UpperBound} and the triggering rule \eqref{State:Push:L1}, it holds
\begin{align*}
&~\frac{{\rm d}\big\|w(t)\big\|_{1}}{{\rm d}t}\\
  \leqslant
&-\sum_{i=1}^{n}\xi_{i}\gamma_{i}(t)\big|w_{i}(t)\big|
 +\sum_{i=1}^{n}\xi_{i}\gamma_{i}(t)\big|e_{i}(t)\big|\\
&+\sum_{i=1}^{n}\xi_{i}a_{ii}^{+}(t)G_{i}\big|w_{i}(t)\big|
 +\sum_{i=1}^{n}\xi_{i}a_{ii}^{+}(t)G_{i}\big|e_{i}(t)\big|\\
&+\sum_{i\neq j}\sum_{j=1}^{n}\xi_{i}\big|a_{ij}(t)\big|G_{j}\big|w_{j}(t)\big|
 +\sum_{i\neq j}\sum_{j=1}^{n}\xi_{i}\big|a_{ij}(t)\big|G_{j}\big|e_{j}(t)\big|\\
%----------------------------------------------------------------------------------------------------------------------------------------------
 \leqslant
&-\sum_{j=1}^{n}\bigg[\gamma_{j}(t)-G_{j}a_{jj}^{+}(t)-G_{j}\sum_{i\neq j}\frac{\xi_{i}}{\xi_{j}}\big|a_{ij}(t)\big|\bigg]\xi_{j}\big|w_{j}(t)\big|\\
&+\sum_{j=1}^{n}\bigg[\gamma_{j}(t)+G_{j}a_{jj}^{+}(t)+G_{j}\sum_{i\neq j}\frac{\xi_{i}}{\xi_{j}}\big|a_{ij}(t)\big|\bigg]\xi_{j}\big|e_{j}(t)\big|
%----------------------------------------------------------------------------------------------------------------------------------------------
\end{align*}
which implies
\begin{align*}
\frac{{\rm d}\big\|w(t)\big\|_{1}}{{\rm d}t}
%---------------------------------------------------------------------------------------------------------------------------------------------
\leqslant&-\mu_{1}(t)\sum_{j=1}^{n}\xi_{j}\big|w_{j}(t)\big|+\M_{1}\sum_{j=1}^{n}\xi_{j}\big|e_{j}(t)\big|\\
\leqslant&-\mu_{1}(t)\big\|w(t)\big\|_{1}+\M_{1}\sum_{j=1}^{n}\xi_{j}\Psi_{j}(t)
%---------------------------------------------------------------------------------------------------------------------------------------------
\end{align*}
By the classical Gronwell inequality, we have
\begin{align*}
\big\|w(t)\big\|_{1}
\leqslant&
~\big\|w(t_{0}^{i})\big\|_{1}{\rm e}^{-\sigma(t,t_{0}^{i})}
+\M_{1}\int_{t_{0}^{i}}^{t}{\rm e}^{-\sigma(t,s)}\big\|\Psi(s)\big\|_{1}{\rm d}s\\
%---------------------------------------------------------------------------------------------------------------------------------------------
=&~{\rm e}^{-\sigma(t,t_{0}^{i})}\bigg[\big\|w(t_{0}^{i})\big\|_{1}
  +\M_{1}\int_{t_{0}^{i}}^{t}{\rm e}^{\sigma(s,t_{0}^{i})}\big\|\Psi(s)\big\|_{1}{\rm d}s\bigg]
\end{align*}
with
\begin{align*}
\sigma(t,s)=\int_{s}^{t}\mu_{1}(\tau)\,{\rm d}\tau
\end{align*}
for $s\in[t_{k}^{i},t]$ and $t\in[t_{k}^{i},t_{k+1}^{i})$. By using the L'Hospital rule, it follows
%\begin{align*}
%  \lim_{t\to+\infty}{\rm e}^{-\sigma(t,t_{0}^{i})}\int_{t_{0}^{i}}^{t}{\rm e}^{\sigma(s,t_{0}^{i})}\big\|\Psi(s)\big\|_{1}{\rm d}s
%  \lim_{t\to+\infty}\frac{1}{{\rm e}^{\sigma(t,t_{0}^{i})}}\int_{t_{0}^{i}}^{t}{\rm e}^{\sigma(s,t_{0}^{i})}\big\|\Psi(s)\big\|_{1}{\rm d}s
%=&\lim_{t\to+\infty}\frac{\big\|\Psi(t)\big\|_{1}{\rm e}^{\sigma(t,t_{0}^{i})}}{\mu_{1}(t)\,{\rm e}^{\sigma(t,t_{0}^{i})}}\\
%\leqslant&\lim_{t\to+\infty}\frac{\big\|\Psi(t)\big\|_{1}}{\m_{1}}\\
%=&~0.
%\end{align*}
%which means
%\begin{align*}
%\lim_{t\to+\infty}\big\|w(t)\big\|_{1}\leqslant
%\lim_{t\to+\infty}\frac{\M_{1}}{{\rm e}^{\sigma(t,t_{0}^{i})}}\int_{t_{0}^{i}}^{t}{\rm e}^{\sigma(s,t_{0}^{i})}\big\|\Psi(s)\big\|_{1}{\rm d}s=0.
%\lim_{t\to+\infty}\M_{1}\int_{t_{0}^{i}}^{t}{\rm e}^{-\sigma(t,s)}\big\|\Psi(s)\big\|_{1}{\rm d}s.
%\end{align*}
\begin{align*}
\lim_{t\to+\infty}\big\|w(t)\big\|_{1}
&\leqslant\lim_{t\to+\infty}\frac{\M_{1}}{{\rm e}^{\sigma(t,t_{0}^{i})}}\int_{t_{0}^{i}}^{t}{\rm e}^{\sigma(s,t_{0}^{i})}\big\|\Psi(s)\big\|_{1}{\rm d}s\\
&=\lim_{t\to+\infty}\M_{1}\frac{\big\|\Psi(t)\big\|_{1}{\rm e}^{\sigma(t,t_{0}^{i})}}{\mu_{1}(t)\,{\rm e}^{\sigma(t,t_{0}^{i})}}\\
&\leqslant\lim_{t\to+\infty}\frac{\M_{1}}{\varepsilon_{d}}\big\|\Psi(t)\big\|_{1}\\
&=0,
\end{align*}
where $t_{0}^{i}=0$ for all $i=\oneton$, which means that $\|w(t)\|_{1}$ converges to $0$ by the controlling time sequence $\{t_{k}^{i}\}_{k=0}^{+\infty}~(i=\oneton)$.
Hence, the system \eqref{Push} achieves out-synchronization and Theorem \ref{State:Push} is proved.
\end{proof}

\begin{proposition}
Let $\Psi(t)=[\Psi_{1}(t),\cdots,\Psi_{n}(t)]^{\top}$ be a vector of positive continuous decreasing functions on $[0,+\infty)$ and $\Psi(0)>0$.
Set $t_{k+1}^{i}$ as the triggering time points such that
\begin{align*}
t_{k+1}^{i}=\max_{\tau\geqslant t_{k}^{i}}\Big\{\tau:\big|e_{i}(t)\big|\leqslant\Psi_{i}(t),~\forall~t\in[t_{k}^{i},\tau)\Big\}
\end{align*}
for $i=\oneton$ and all $k=\zerotoinfty$. If $\mu_{2}(t)\geqslant\varepsilon_{d}$ for some $\varepsilon_{d}>0$ and
\begin{align*}
\lim_{t\to+\infty}\Psi_{i}(t)=0
\end{align*}
for $i=\oneton$, then the system \eqref{Push} reaches out-synchronization.
\end{proposition}

\begin{proposition}
Let $\Psi(t)=[\Psi_{1}(t),\cdots,\Psi_{n}(t)]^{\top}$ be a vector of positive continuous decreasing functions on $[0,+\infty)$ and $\Psi(0)>0$.
Set $t_{k+1}^{i}$ as the triggering time points such that
\begin{align*}
t_{k+1}^{i}=\max_{\tau\geqslant t_{k}^{i}}\Big\{\tau:\big|e_{i}(t)\big|\leqslant\Psi_{i}(t),~\forall~t\in[t_{k}^{i},\tau)\Big\}
\end{align*}
for $i=\oneton$ and all $k=\zerotoinfty$. If $\mu_{\infty}(t)\geqslant\varepsilon_{d}$ for some $\varepsilon_{d}>0$ and
\begin{align*}
\lim_{t\to+\infty}\Psi_{i}(t)=0
\end{align*}
for $i=\oneton$, then the system \eqref{Push} reaches out-synchronization.
\end{proposition}

\begin{remark}
The preliminary condition for systems \eqref{Centralized} and \eqref{Push} to achieve out-synchronization is that the existing duration of the solution in the Cauchy problem of systems \eqref{Discrete:Centralized} and \eqref{Discrete:Push} should be $[0,+\infty)$ (equivalently $\lim_{k\to\infty}t_{k}=+\infty$ and $\lim_{k\to\infty}t_{k}^{i}=+\infty$ for all $i=\oneton$). The verification of this condition will be given in Section \ref{Zeno}.
\end{remark}

\begin{remark}
To loosen the assumption that $\mu_{m}(t)\geqslant\varepsilon_{d}~~(m=1,2\text{~or~}\infty)$ for some $\varepsilon_{d}>0$ and $\lim_{t\to+\infty}\Psi_{i}(t)=0$, the function $\Psi_{i}(t)$ can be designed as follow
\begin{align*}
\delta_{k}(t)\,e^{-\beta_{i}(t-t_{k}^{i})}
\end{align*}
where
\begin{align*}
\delta_{k}(t)=\inf_{i=\oneton}\left\{\frac{\alpha\big\|w(t)\big\|_{1}}{\sum\limits_{i=1}^{n}e^{-\beta_{i}(t-t_{k}^{i})}}:t-t_{k}^{i}\in[0,T]\right\}
\end{align*}
and
\begin{align*}
\delta_{k}(t)=\inf_{i=\oneton}\left\{\frac{\alpha\big\|w(t)\big\|_{2}}{\sqrt{\sum\limits_{i=1}^{n}\xi_{i}\,e^{-2\beta_{i}(t-t_{k}^{i})}}}:t-t_{k}^{i}\in[0,T]\right\}
\end{align*}
and
\begin{align*}
\delta_{k}(t)=\inf_{i=\oneton}\left\{\frac{\alpha\,\xi_{i^{*}}\big\|w(t)\big\|_{\infty}}{\max\limits_{i=\oneton}\big\{e^{-\beta_{i}(t-t_{k}^{i})}\big\}}:t-t_{k}^{i}\in[0,T]\right\}.
\end{align*}
which depend on the global state information $\|w(t)\|_{1}$, $\|w(t)\|_{2}$ and $\|w(t)\|_{\infty}$.
In these design, the out-synchronization of system \eqref{Push} can be proved and the Zeno behavior can also be excluded without other restricted conditions.
\end{remark}

\section{Exclusion of Zeno behavior\label{Zeno}}

In this section, we are to prove the absence of Zeno behavior. To this aim, we will find a common positive lower bound for the inter-event time $t_{k+1}-t_{k}$ or $t_{k+1}^{i}-t_{k}^{i}$, for all the neurons $i=\oneton$ and $k=\zerotoinfty$.

\begin{theorem}\label{Zeno:AlternateTriggering}
Under either the centralized data-sampling rule in Theorem \ref{State:Centralized} or the push-based decentralized data-sampling rule in Theorem \ref{State:Push}, the inter-event interval of every neuron is strictly positive and has a common positive lower bound. Moreover, the Zeno behavior is excluded.
\end{theorem}

\begin{proof}
(1) For the centralized rule, consider the following derivative.
\begin{align*}
 &~\frac{{\rm d}}{{\rm d}t}\frac{\big\|e(t)\big\|_{1}}{\big\|w(t)\big\|_{1}}
  =\frac{\frac{{\rm d}}{{\rm d}t}\big\|e(t)\big\|_{1}}{\big\|w(t)\big\|_{1}}
  -\frac{\big\|e(t)\big\|_{1}}{\big\|w(t)\big\|_{1}}
   \frac{\frac{{\rm d}}{{\rm d}t}\big\|w(t)\big\|_{1}}{\big\|w(t)\big\|_{1}}\\
%----------------------------------------------------------------------------------------------------
=&-\frac{\sum\limits_{i=1}^{n}\xi_{i}\sign\big(e_{i}(t)\big)\dot{w}_{i}(t)}{\big\|w(t)\big\|_{1}}
  -\frac{\sum\limits_{i=1}^{n}\xi_{i}\sign\big(w_{i}(t)\big)\dot{w}_{i}(t)}{\big\|w(t)\big\|_{1}}
   \frac{\big\|e(t)\big\|_{1}}{\big\|w(t)\big\|_{1}}\\
\leqslant&
  ~\Bigg[\mu_{1}(t)+\nu_{1}(t)\frac{\big\|e(t)\big\|_{1}}{\big\|w(t)\big\|_{1}}\Bigg]
   \Bigg[1+\frac{\big\|e(t)\big\|_{1}}{\big\|w(t)\big\|_{1}}\Bigg]\\
\leqslant&
  ~\Bigg[\m_{1}+\M_{1}\frac{\big\|e(t)\big\|_{1}}{\big\|w(t)\big\|_{1}}\Bigg]
   \Bigg[1+\frac{\big\|e(t)\big\|_{1}}{\big\|w(t)\big\|_{1}}\Bigg]\\
\leqslant&
  ~\M_{1}\Bigg[1+\frac{\big\|e(t)\big\|_{1}}{\big\|w(t)\big\|_{1}}\Bigg]^{2},
\end{align*}
where $\m_{1}=\sup_{t\in[0,+\infty)}\{\mu_{1}(t)\}$ and $\m_{1}\leqslant\M_{1}$. Via comparison principle, we have
\begin{align*}
\frac{\big\|e(t)\big\|_{1}}{\big\|w(t)\big\|_{1}}\leqslant\phi(t)
\end{align*}
where $\phi(t)$ is the solution of the following differential equation
\begin{align*}
\begin{cases}
\dfrac{{\rm d}\phi(t)}{{\rm d}t}=\M_{1}\Big[\phi(t)+1\Big]^{2}\\[7pt]
\phi(t_{0})=\phi_{0}
\end{cases}.
\end{align*}
Hence the inter-event time $t_{k+1}-t_{k}$ has a common lower bound $\eta_{c}$ which follows
\begin{align*}
\eta_{c}=\dfrac{\phi_{0}}{\m_{1}(1+\phi_{0})}.
\end{align*}
For the lower bound $\eta_{c}$ is uniform for all the neurons, the next triggering time point $t_{k+1}$ satisfies $t_{k+1}\geqslant t_{k}+\eta_{c}$ for all $i=\oneton$ and $k=\zerotoinfty$. Therefore we can assert that there is no Zeno behavior for all the neurons.

(2) For the push-based decentralized rule, let us consider the following derivative of the state measurement error for any neuron $v_{i}~(i=\oneton)$.
\begin{align*}
 \big\|e(t)\big\|_{1}\leqslant
 &\int_{t_{k}^{i}}^{t}\big\|\dot{w}(s)\big\|_{1}{\rm d}s\\
=&\int_{t_{k}^{i}}^{t}\sum_{j=1}^{n}
  \xi_{j}\bigg|-\gamma_{j}(s)w_{j}(t_{k}^{j})+\sum_{l=1}^{n}a_{jl}(s)h_{l}\big(w_{l}(t_{k}^{l})\big)\bigg|{\rm d}s\\
%--------------------------------------------------------------------------------------------------------------------------
  \leqslant
 &\int_{t_{k}^{i}}^{t}\sum_{l=1}^{n}
  \bigg[\gamma_{l}(t)+\sum_{j=1}^{n}\frac{\xi_{j}}{\xi_{l}}\big|a_{jl}(t)\big|G_{l}\bigg]
  \xi_{l}\big|w_{l}(t_{k}^{l})\big|\,{\rm d}s\\
%--------------------------------------------------------------------------------------------------------------------------
  \leqslant
 &\sum_{l=1}^{n}\xi_{l}\big|w_{l}(t_{k}^{l})\big|\int_{t_{k}^{i}}^{t}\M_{1}\,{\rm d}s\\
  \leqslant
 &~\M_{1}\big\|w(0)\big\|_{1}(t-t_{k}^{i}),
\end{align*}
where $\|w(0)\|_{1}$ is a given positive constant. Based on the triggering rule \eqref{State:Push:L1}, the event will not trigger until $|e_{i}(t)|=\Psi_{i}(t)$ at time point $t=t_{k+1}^{i}$. Hence, it holds
\begin{align*}
\Psi_{i}(t_{k+1}^{i})=\M_{1}\big\|w(0)\big\|_{1}(t_{k+1}^{i}-t_{k}^{i}).
\end{align*}

%{\color{red}[[$\le$ instead of $=$ here!!]]}

Given any positive time point $T>0$, suppose that there is at least one neuron $i$ exhibiting the Zeno behavior on the finite time period $[0,T]\subset[0,+\infty)$, that is, there exist infinite number of triggering on $[0,T]$. Then it satisfies
\begin{align*}
\lim_{k\to+\infty}t_{k}^{i}=t^{\bm*}\in[0,T].
\end{align*}
Since $\Psi_{i}(t)$ is a continuous function on $[t_{0},+\infty)$, we have
\begin{align*}
 \Psi_{i}(t^{\bm*})
&=\lim_{k\to+\infty}\Psi_{i}(t_{k+1}^{i})\\
&=\lim_{k\to+\infty}\M_{1}\big\|w(0)\big\|_{1}(t_{k+1}^{i}-t_{k}^{i})\\
&=0,
\end{align*}
which means there exists a time point $t^{\bm*}\in[0,T]$ such that $\Psi_{i}(t^{\bm*})=0$.
This contradicts that $\Psi_{i}(t)~(i=\oneton)$ is a positive function on $[0,+\infty)$.
Therefore, for the arbitrariness of $T>0$, we can assert that there is no Zeno behavior for all the neurons on $[0,+\infty)$. That is to say, the next inter-event interval has a common positive lower bound for each neuron $i=\oneton$, which satisfies
\begin{align*}
\eta_{d}=\min_{i=\oneton}\Big\{\eta_{d}^{i}:\Psi_{i}(\eta_{d}^{i}+t_{k}^{i})=\M_{1}\big\|w(0)\big\|_{1}\eta_{d}^{i}\Big\}
\end{align*}
This completes the proof.
\end{proof}
\begin{remark}
The proof for Theorem \ref{Zeno:AlternateTriggering} is given under $l_{1}$ norm. By similar approach, one can also prove the theorem via $l_{2}$ and $l_{\infty}$ norm.
\end{remark}

After proving the exclusion of Zeno behavior, we are at the stage to conclude that $\lim_{k\to+\infty}t_{k}=+\infty$ and $\lim_{k\to+\infty}t_{k}^{i}=+\infty$ for all $i=\oneton$. This implies the following result.
\begin{theorem}
Under the data-sampling rule described in Theorem \ref{State:Centralized} and Theorem \ref{State:Push}, existing duration of the solution in the Cauchy problem of systems \eqref{Discrete:Centralized} and \eqref{Discrete:Push} are all $[0,+\infty)$, equivalently
\begin{align*}
\lim_{k\to\infty}t_{k}=+\infty
\end{align*}
and
\begin{align*}
\lim_{k\to\infty}t_{k}^{i}=+\infty
\end{align*}
for all $i=\oneton$.
\end{theorem}

\begin{proof}
In fact, from Theorem \ref{Zeno:AlternateTriggering}, one can see that for all initial values, the trajectory of systems \eqref{Discrete:Centralized} and \eqref{Discrete:Push} possess discontinuous triggerring events with two positive lower bounds $\eta_{c}$ and $\eta_{d}$ of inter-event time respectively. This implies that $\lim_{k\to\infty}t_{k}=+\infty$ or $\lim_{k\to+\infty}t_{k}^{i}=+\infty$ for all the neuron $i=\oneton$. Therefore, the solutions in the Cauchy problem of systems \eqref{Discrete:Centralized} and \eqref{Discrete:Push} all exist for the duration $[0,+\infty)$.
\end{proof}

\section{Numerical simulation}
In this section, we provide a numerical example to illustrate the theoretical results. The comparisons between the centralized and push-based decentralized rules based on both structure and state are also given. Let us consider the switching topologies $\mathcal S=\{(\Gamma_{1},A_{1},I_{1}),\cdots,(\Gamma_{6},A_{6},I_{6})\}$, where
\begin{align*}
&\Gamma_{1}=\text{\it diag}\,\big\{0.8850, ~0.9148, ~0.8530, ~0.7977, ~0.8764\big\}\\
&\Gamma_{2}=\text{\it diag}\,\big\{0.7484, ~0.9326, ~0.6340, ~0.9843, ~0.5494\big\}\\
&\Gamma_{3}=\text{\it diag}\,\big\{0.7735, ~0.7015, ~0.8535, ~0.8621, ~0.9068\big\}\\
&\Gamma_{4}=\text{\it diag}\,\big\{0.8915, ~0.7833, ~0.9057, ~0.7884, ~0.9720\big\}\\
&\Gamma_{5}=\text{\it diag}\,\big\{0.9357, ~0.7538, ~0.8944, ~0.7365, ~0.9144\big\}\\
&\Gamma_{6}=\text{\it diag}\,\big\{0.6612, ~0.9881, ~0.6391, ~0.5364, ~0.8756\big\}
\end{align*}
and
\begin{align*}
&A_{1}=\setlength{\arraycolsep}{4pt}
\begin{bmatrix}
  -1.7919 &   -0.3948 & ~~~0.2564 &   -0.3204 &   -0.0156\\
~~~0.4671 & ~~~0.2490 &   -0.7117 &   -0.1370 &   -0.0501\\
  -0.7011 & ~~~0.0369 &   -1.8727 &   -0.7410 &   -0.0184\\
  -0.1982 & ~~~0.1655 & ~~~0.8427 & ~~~0.3652 &   -0.7693\\
~~~0.6181 & ~~~0.5135 &   -0.5559 & ~~~0.0658 &   -1.9569
\end{bmatrix}\\
&A_{2}=\setlength{\arraycolsep}{4pt}
\begin{bmatrix}
~~~0.2630 &   -0.3615 & ~~~0.8626 & ~~~0.3302 & ~~~0.2694\\
~~~0.4676 &   -1.8345 &   -0.5973 &   -0.4837 &   -0.3797\\
  -0.8931 & ~~~0.0360 &   -1.7021 &   -0.1515 &   -0.8251\\
  -0.0750 &   -0.3230 & ~~~0.5239 &   -1.9542 &   -0.2013\\
  -0.1842 &   -0.0325 & ~~~0.2393 & ~~~0.3162 & ~~~0.2926
\end{bmatrix}\\
%------------------------------------
&A_{3}=\setlength{\arraycolsep}{4pt}
\begin{bmatrix}
~~~0.3798 &   -0.5099 &   -0.4776 & ~~~1.4789 & ~~~0.8120\\
  -0.4506 &   -1.7393 & ~~~0.2600 & ~~~0.4094 & ~~~0.1505\\
~~~0.3564 & ~~~0.5781 &   -1.6185 & ~~~0.2230 & ~~~0.2439\\
~~~0.1150 & ~~~0.4990 &   -0.1876 &   -1.6549 &   -0.6292\\
~~~0.2979 & ~~~0.4720 &   -0.2338 &   -0.6050 & ~~~0.8528
\end{bmatrix}\\
&A_{4}=\setlength{\arraycolsep}{4pt}
\begin{bmatrix}
  -1.7522 &   -0.0166 & ~~~0.3873 & ~~~0.0970 &   -1.1968\\
  -0.3338 &   -1.8286 & ~~~0.3803 &   -0.5127 & ~~~0.7253\\
  -0.1573 & ~~~0.4312 & ~~~0.4020 &   -0.5886 &   -0.6525\\
~~~0.0200 &   -0.7156 &   -0.6737 &   -1.0330 &   -0.5318\\
  -0.1808 & ~~~0.4284 & ~~~0.2678 &   -0.0480 &   -1.3318
\end{bmatrix}\\
%-------------------------------------
&A_{5}=\setlength{\arraycolsep}{4pt}
\begin{bmatrix}
  -1.8018 &   -0.5470 &   -0.1406 & ~~~0.2769 &   -0.8000\\
  -0.4222 & ~~~0.2530 & ~~~0.4295 & ~~~0.5383 & ~~~0.1825\\
  -0.7572 &   -0.4001 &   -1.9090 & ~~~0.6196 & ~~~0.6523\\
  -0.7861 & ~~~0.5978 & ~~~0.2121 &   -1.5166 & ~~~0.2531\\
~~~0.1823 & ~~~0.5187 &   -0.2007 & ~~~0.3803 & ~~~0.1668
\end{bmatrix}\\
&A_{6}=\setlength{\arraycolsep}{4pt}
\begin{bmatrix}
  -1.6122 &   -0.4175 &   -0.4285 & ~~~0.5557 & ~~~0.4177\\
~~~0.1750 & ~~~0.6452 & ~~~0.2641 &   -0.1387 &   -0.4541\\
~~~0.6864 &   -0.1068 &   -1.0629 &   -0.1994 & ~~~0.1796\\
~~~0.4106 & ~~~0.2553 &   -0.7769 & ~~~0.7958 & ~~~0.5536\\
~~~0.2599 &   -0.1512 & ~~~0.1097 &   -0.3196 &   -1.5582
\end{bmatrix}
\end{align*}
and
\begin{align*}
&I_{1}=\begin{bmatrix}~~~0.6353\\~~~0.5897\\~~~0.2886\\  -0.2428\\~~~0.6232\end{bmatrix}\hspace{1ex}
 I_{2}=\begin{bmatrix}~~~0.0657\\  -0.2985\\~~~0.8780\\~~~0.7519\\~~~0.1003\end{bmatrix}\hspace{1ex}
 I_{3}=\begin{bmatrix}~~~0.2450\\~~~0.1741\\  -0.5845\\  -0.3975\\  -0.0582\end{bmatrix}\\[1pt]
&I_{4}=\begin{bmatrix}  -0.5390\\~~~0.6886\\  -0.6105\\  -0.5482\\  -0.6586\end{bmatrix}\hspace{1ex}
 I_{5}=\begin{bmatrix}  -0.5447\\  -0.1286\\  -0.3778\\~~~0.8468\\  -0.1396\end{bmatrix}\hspace{1ex}
 I_{6}=\begin{bmatrix}  -0.6304\\~~~0.8098\\~~~0.9595\\  -0.1223\\  -0.7778\end{bmatrix}
\end{align*}
The activation function satisfies $g_{i}(u_{i})=1/(1+{\rm e}^{-u_{i}})$ and the switching time sequence of $(A_{i},\Gamma_{i},I_{i})$ follows a Poisson process with $\lambda=1$. In the structure-dependent rules, we set $\alpha=0.2$, $\beta_{i}=1~(i=\oneton)$, $T=500$, $\epsilon_{c}=\varepsilon_{c}=0.01$ and $\epsilon_{d}=\varepsilon_{d}=0.02$. In state-dependent rules, the function $\Phi(t)$ in centralized data-sampling is given as
\begin{align*}
\Phi(t)=\frac{8000}{\big(0.0065\,t+6.5\big)^{5}}
\end{align*}
and the functions $\Psi_{i}(t)~(i=\oneton)$ in push-based decentralized data-sampling are given as follow
\begin{align*}
&\Psi_{1}(t)=\frac{27000}{\big(0.007\,t+0.68\big)^{6}},&
&\Psi_{2}(t)=\frac{90000}{\big(0.01\,t+1.27\big)^{6}},\\
&\Psi_{3}(t)=\frac{80000}{\big(0.012\,t+1.02\big)^{6}},&
&\Psi_{4}(t)=\frac{(t+100)\,{\rm e}^{-0.01\,t-1}}{700\,\Gamma(2)},\\
&\Psi_{5}(t)=\frac{2100}{\big(0.005\,t+0.5\big)^{6}},&
\end{align*}
where $\Gamma(n)$ is a gamma function.

Figures \ref{Trajectory:CentralizedStructure}, \ref{Trajectory:PushStructure}, \ref{Trajectory:CentralizedState} and \ref{Trajectory:PushState} show that the two trajectories $u_{1}(t)$ and $v_{1}(t)$ of neuron $1$ starting from two different initial values converge to each other. Figures \ref{Error:L1}, \ref{Error:L2} and \ref{Error:LInfty} plot the logarithm of the error dynamics $\log\|u(t)-v(t)\|$ in four data-sampling rules under three norms $l_{1}$, $l_{2}$ and $l_{\infty}$, which implies that the convergence is exponential. Figure \ref{StatisticalResult} describes the statistical results of the triggering time points. We can see that the number of the triggering time points $\{t_{k}\}_{k=0}^{+\infty}$ or $\{t_{k}^{i}\}_{k=0}^{+\infty}~(\oneton)$ in structure-dependent rules is more than that in state-dependent rules, and the average number of the triggering time points $\langle t_{k}^{i}\rangle_{i}$ over the five neurons in push-based decentralized rules is more than the number of triggering time points $t_{k}$ in centralized rules.

\section{Conclusions}

In this paper, the out-synchronization dynamics of both centralized and push-based decentralized asymmetrical time-varying neural network by data-sampling are discussed.
The sufficient conditions for both sampled-data rules are proposed and proved to guarantee the global out-synchronization.
In addition, the exclusion of the Zeno behavior can be verified by proving the common lower-bounds of the time-varying time-steps.
One numerical example is provided to illustrate our theoretical results.
The results can also be extended to the case where the underlying systems involve stochastic disturbances or controlled in networked environment. \citep{Ding,Shen} may be the starting point  of our future extension on this issue.

\bibliographystyle{elsarticle-harv}
%\bibliography{Reference}

\begin{figure}[H]
\centering
\subfigure[The trajectory of neuron No.$1$ starting from two different initial values in structure-dependent centralized system.]
{\label{Trajectory:CentralizedStructure}\includegraphics[width=0.485\textwidth]{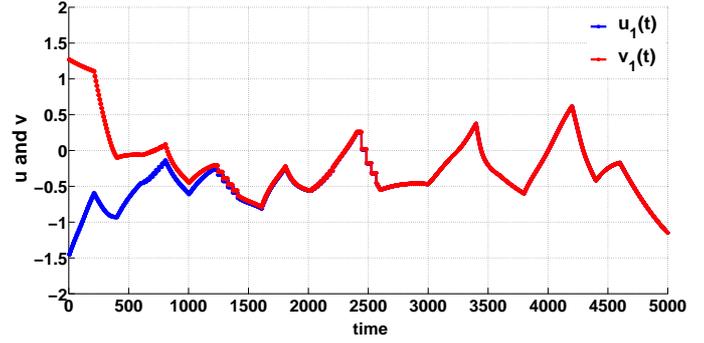}}\\
%-----------------------------------------------------------------------------------------------------------------------------
\subfigure[The trajectory of neuron No.$1$ starting from two different initial values in structure-dependent push-based decentralized system.]
{\label{Trajectory:PushStructure}\includegraphics[width=0.485\textwidth]{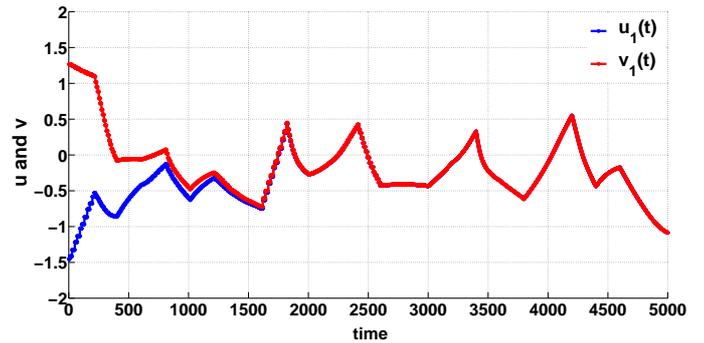}}\\
%-----------------------------------------------------------------------------------------------------------------------------
\subfigure[The trajectory of neuron No.$1$ starting from two different initial values in state-dependent centralized system.]
{\label{Trajectory:CentralizedState}\includegraphics[width=0.485\textwidth]{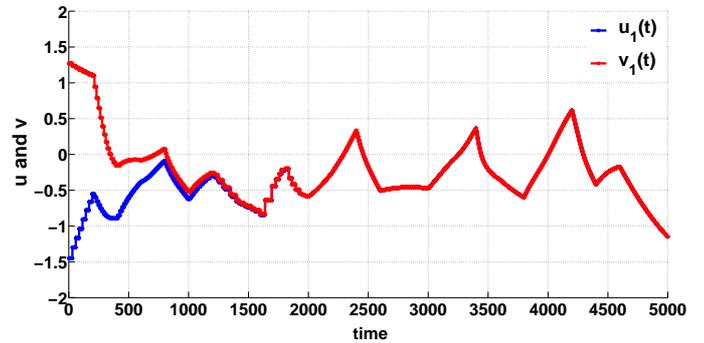}}\\
%-----------------------------------------------------------------------------------------------------------------------------
\subfigure[The trajectory of neuron No.$1$ starting from two different initial values in state-dependent push-based decentralized system.]
{\label{Trajectory:PushState}\includegraphics[width=0.485\textwidth]{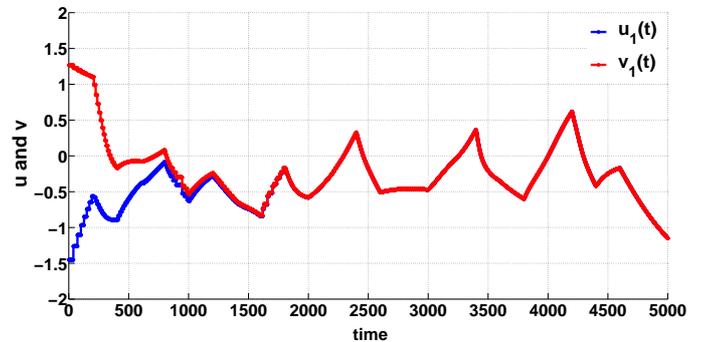}}
%-----------------------------------------------------------------------------------------------------------------------------
\caption{The figures show the two trajectories of neuron No.$1$ for two different initial values, which indicates the out-synchronization of the systems.}
\label{Trajectory}
\end{figure}

\begin{figure}[H]
\centering
\subfigure[The logarithm of the error dynamics $\log\|u(t)-v(t)\|_{1}$ in four data-sampling rules under $l_{1}$- norm.]
{\label{Error:L1}\includegraphics[width=0.49\textwidth]{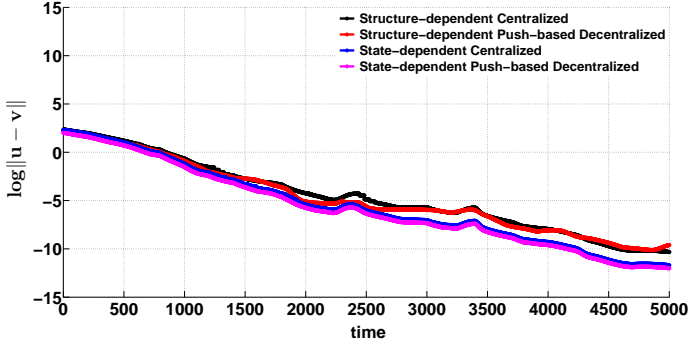}}\\
%-----------------------------------------------------------------------------------------------------------------------------
\subfigure[The logarithm of the error dynamics $\log\|u(t)-v(t)\|_{2}$ in four data-sampling rules under $l_{2}$- norm.]
{\label{Error:L2}\includegraphics[width=0.49\textwidth]{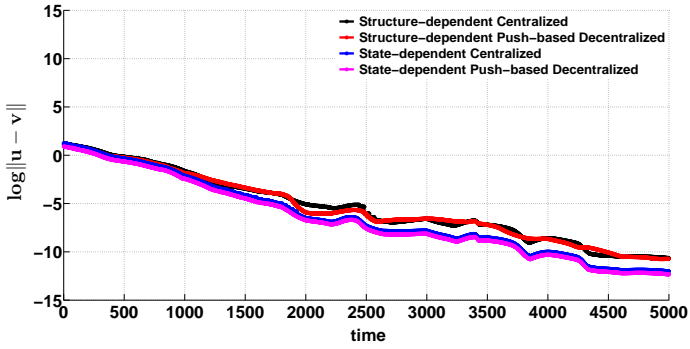}}\\
%-----------------------------------------------------------------------------------------------------------------------------
\subfigure[The logarithm of the error dynamics $\log\|u(t)-v(t)\|_{\infty}$ in four data-sampling rules under $l_{\infty}$- norm.]
{\label{Error:LInfty}\includegraphics[width=0.49\textwidth]{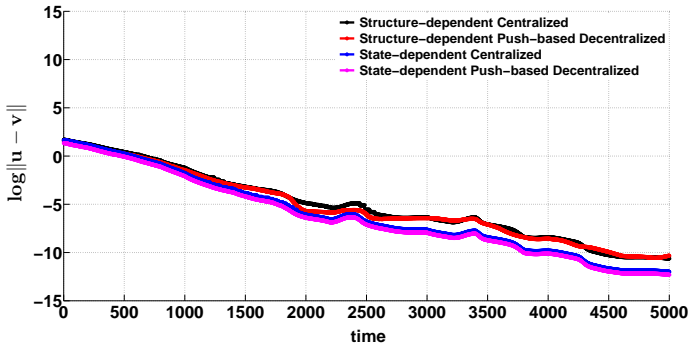}}\\
%-----------------------------------------------------------------------------------------------------------------------------
\subfigure[The statistical results in the number of the triggering time points during each time period by four different rules. The magenta lines show the maximum and minimum number of the triggering time points in five neurons.]
{\label{StatisticalResult}\includegraphics[width=0.49\textwidth]{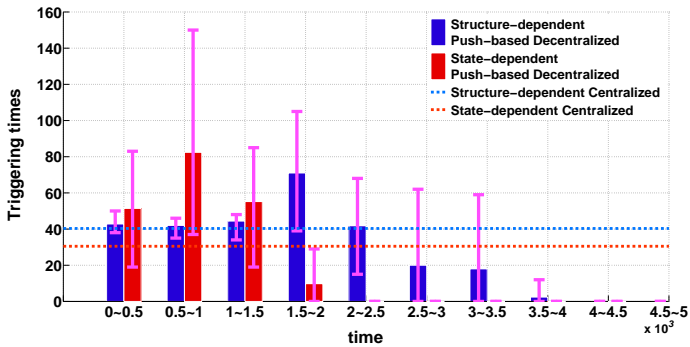}}
%-----------------------------------------------------------------------------------------------------------------------------
\caption{The figures show the logarithm of the error dynamics and the statistical results in the number of the triggering time points.}
\label{Error}
\end{figure}

\end{document}